\documentclass[twoside,11pt]{article}
\usepackage{jmlr2e}

\usepackage{natbib}
\usepackage{amsmath,amssymb}
\usepackage{url}
\usepackage{color}
\usepackage{soul}

 \newtheorem{arrangement}[theorem]{Arrangement}

\begin{document}
\title{Bounds for the VC Dimension of 1NN Prototype Sets}

\author{\name Iain A. D. Gunn \email iain.a.d.gunn@gmail.com\\
	\name Ludmila~I.~Kuncheva \email l.i.kuncheva@bangor.ac.uk\\
	\addr School of Computer Science,\\
	      Bangor University,\\
	      Dean Street,\\
	      Bangor, Gwynedd,\\
	      Wales LL57 2NJ,\\
	      UK}
	
\editor{N/A}

\maketitle

\begin{abstract}
In Statistical Learning, the Vapnik-Chervonenkis (VC) dimension is an important combinatorial property of classifiers.  To our knowledge, no theoretical results yet exist for the VC dimension of edited nearest-neighbour (1NN) classifiers with reference set of fixed size.  Related theoretical results  are scattered in the literature and their implications have not been made explicit.  We collect some relevant results and use them to provide explicit lower and upper bounds for the VC dimension of 1NN classifiers with a prototype set of fixed size. We discuss the implications of these bounds for the size of training set needed to learn such a classifier to a given accuracy. Further, we provide a new lower bound for the two-dimensional case, based on a new geometrical argument.
\end{abstract}


\section{Introduction}
\label{introduction}

The VC dimension is a measure of what is called the ``capacity'' of a classification algorithm: that is, roughly speaking, its flexibility or expressive power.  The practical interest of this quantity arises from its role in the \emph{Probably Approximately Correct} learning model and related models (see e.g. \citet[ch. 3]{Shais}, \citet{Holden95}), where it appears in results relating the size of the training set to the classifier's accuracy.  The VC dimension is defined in terms of the concept of \emph{shattering}~\citep{Vapnik98}. For a two-class problem in $d$-dimensional space $\mathbb{R}^d$, a set $A$ of functions $\mathbb{R}^d \to \{+1,-1\}$ is said to shatter a set of points (in $\mathbb{R}^d$) if any labelling of that set can be given by an element of $A$.  The VC dimension of $A$ is the size of the largest set which can be shattered by $A$.

The nearest-neighbour (1NN) classification rule is a classic technique of supervised learning, first introduced by \cite{Fix52}.  The 1NN rule determines a label for (``classifies'') a given point in a metric space by assigning it the label of the nearest point in a previously determined set of labelled reference points. In the original and simplest algorithm, the reference set is the set of all the training data.  If the amount of training data is \emph{a priori} unbounded, then the VC dimension of the set of associated 1NN-rule classifying functions is infinite: trivially, a set of points of any size is labelled correctly by a classifier whose reference set is the same set of labelled points.

However, it is often not practical to store all the training data, especially in an era of ``big data''.  Therefore  many algorithms have been proposed which learn a smaller reference set from the training data: \cite{Garcia12} and \cite{Triguero12} survey more than 75 such algorithms between them.  What is the VC dimension associated with these ``editing'' NN algorithms?  If the reference set may grow without bound, then the VC dimension is infinite, as for the na{\"i}ve classifier described above.  But if the reference set is constrained not to exceed a given size, then the VC dimension is finite. The main results of the present work are lower and upper bounds for the VC dimension of the set of all 1NN-rule classifiers which use a reference set of given fixed size, in Euclidean space.  

It should be noted that, in general, the size of the reference set to be formed by an editing algorithm is not fixed \textit{a priori}.  There are exceptions, an important one being the case where a classifier for a data stream is kept current by using a fixed window of the most recent $N$ points as its reference set \citep{GunnNeurocomputing}.  But it is very common to impose a \emph{maximum} size limit on the reference set, if only by hand in an \textit{ad hoc} fashion.  Our upper bounds will apply to any algorithm for which the size of the reference set is bounded, though our lower bounds will apply only where the size is prescribed, and the algorithm is capable of returning any reference set of that size.

We believe that there is significant interest in the VC dimension of prototype classifiers with  fixed-size reference sets, and that this is shown by the fact that an incorrect purported result \cite[Proposition 2]{Karacali} has been cited more than 50 times as giving the VC dimension of such classifier sets.  Our study corrects the record.

We feel there is significant value in bringing together the existing theoretical results, some of which seem otherwise likely to remain obscure to practitioners.  Beyond our deductions from existing theory, our novel contributions are 1) A new lower bound for the two-dimensional case, higher than that implied by previous results for polytopes (Proposition \ref{OneBetter}),  and 2) an upper bound for all dimensions, slightly less tight than the best that can be deduced from existing theory, but avoiding the use of exotic functions and thereby facilitating a discussion of the asymptotic behaviour of the limits (Corollary \ref{LooseCor}).  

We will discuss the theoretical framework in section \ref{NNinVC}, and briefly review the relevant literature in section \ref{lit}. In sections \ref{lower2d} and \ref{lowerdd} we derive lower bounds for the VC dimension, and in section \ref{upper} we determine upper bounds.  Our results are summarised in section \ref{conclusions}.

\section{1NN classifiers in the VC theory}
\label{NNinVC}

We use ``classifier'' to mean a function from the feature space to the set of classes, which is used to classify unlabelled examples. An algorithm which learns such a function from training data is a ``classification algorithm''. The set of all classifiers which the algorithm might produce in response to all sets of training data is called the ``hypothesis class'' of that algorithm.  (For example, the original Rosenblatt Perceptron selects among all functions which map one half-space to one class, and the other half-space to the other class: this set of functions is the hypothesis class of the Perceptron.)  When we speak of the VC dimension of a classification algorithm, we mean the VC dimension of the hypothesis class from which that algorithm learns a classifier.


NN classification algorithms 
are not usually thought of as selecting a classifier from a fixed hypothesis class in this way~\citep[see e.g.][ch.\ 19]{Shais}. 
This is because the great practical advantage of 1NN classification is that a classifier function does not need to be explicitly evaluated when classifying an unlabelled example; the new point is simply assigned the class of the nearest prototype in the reference set, which can be identified efficiently using a $k$-d tree.  However, this process is equivalent to classifying the new point according to a classifier function which maps the Voronoi cell of each prototype to the label of that prototype.  Thus, the hypothesis class of the 1NN rule with a reference set of $m$ prototypes in $d$-dimensional space is the set of all labellings of all $m$-cell Voronoi diagrams in the space; it is parameterised by the $md$ co-ordinates of the $m$ prototypes, and the $m$ choices of label.  

We consider only features which take real values, thus classifiers whose domain is $\mathbb{R}^d$, and consider only the 1NN rule with the Euclidean metric. $\mathrm{1NN}(d,m)$ will denote the set of all classifiers $g:\mathbb{R}^d \to \{+1,-1\}$ which use the (Euclidean) nearest-neighbour rule with a reference set of size $m$.  $\mathrm{VCdim}(A)$ will denote the VC dimension of a set $A$ of classifiers.  The purpose of this paper is to give lower and upper bounds for $\mathrm{VCdim}(\mathrm{1NN}(d,m))$.



\section{Related Work}
\label{lit}

A number of existing theoretical results have implications for the VC dimension of the NN classifier with arbitrary reference set of fixed size.  However, these results are scattered in the literature and their implications have not been made explicit. In particular, the results we use to establish lower bounds were developed for a class of polytope classifiers, without mention of NN classifiers. In the present section we give a brief overview of relevant theoretical work, starting with some brief historical context and going on to include the work whose implications we will directly use in subsequent sections.

Questions of ``separating capacities'' of families of decision surfaces were already considered before the advent of the Vapnik-Chervonenkis theory (see e.g.\ \cite{Cover65} and references therein).  During and shortly after the years of the initial development of the VC theory, several authors used approaches from classical combinatorial geometry to derive results about the VC dimension (or related separability properties) of several simple sets.  Readers interested in this literature will need to be aware of the distinction between the VC dimension of a set of classifiers, as we have defined it above, and the VC dimension of a set of subsets of the Euclidean space in question (called by some authors a ``concept class'').  See our discussion at the start of section \ref{upper}, and~\citet[pp. 196, 199, and 215]{Devroye96}.

To give an example of the results obtained, \citet{Dudley79} reports that the set of balls in $\mathbb{R}^d$ has VC dimension $d+1$.  (N.B. Dudley's quantity $V$ is one greater than the quantity defined as the VC dimension in more recent literature.)  Similarly, it may be shown that the set of all half-spaces in $\mathbb{R}^d$ has VC dimension $d+1$~\citep[see e.g.][ch. 13]{Devroye96}.

More recent authors have considered the intersection or union of half-spaces, which is to say, sets which are the interior or exterior of (possibly unbounded) polytopes. \citet{Blumer89} show that the set of interiors of $N$-gons has VC dimension $2N + 1$.  This result is quoted by \citet{TakacsPataki} as the starting point for their work in which they find upper and lower bounds for the VC dimension of convex polytope classifiers, from which we will derive a lower bound in section~\ref{lowerdd}.

NN classifiers, like convex polytope classifiers, have decision boundaries which are the union of subsets of hyperplanes.  However, the decision region for a NN classifier is not in general a simple intersection or union of half-spaces; it may be a complicated union of the interiors of polytopes formed by such intersections.  This may explain why the question of the VC dimension of the NN classifier with an arbitrary reference set of fixed size $m$ has not previously, to our knowledge, been addressed, other than in~\cite{Karacali}.  (\citet{Devroye96} consider the closely related property of the \emph{shatter coefficient} of such classifiers; we will make explicit the implications of this work in section \ref{upper}.)

\citet[Proposition 2]{Karacali} claim that the VC dimension of the NN classifier with reference set of size $m$, $\mathrm{VCdim}(\mathrm{1NN}(d,m))$ in our notation, is exactly $m$. Their argument consists of exhibiting a set of $m+1$ points which cannot be correctly classified by $m$ prototypes. They do not argue that a general set of $m+1$ points cannot be shattered.  This apparently reflects a misunderstanding of the definition of VC dimension.  The VC dimension is the largest number for which some set of that size can be shattered, not the largest number such that all sets of that size can be shattered.  The latter quantity is called the Popper dimension, a quantity 
which thus far has not found a r\^{o}le in statistical learning theory~\citep{Corfield09}.

\section{Lower bounds - two dimensions}
\label{lower2d}
Tak{\'a}cs and Pataki \citep{Takacs,TakacsPataki} prove bounds for the VC dimension of sets of classifiers whose decision boundaries are convex polytopes.  These sets can easily be related to sets of nearest-neighbour classifiers, giving lower bounds for the latter.  This is because a decision boundary which is an $N$-faceted convex polytope can be obtained as the decision boundary of a 1NN classifier with $N+1$ prototypes, by placing a prototype of one label inside the polytope, and the remaining $N$ prototypes, with the opposite label, as the reflections of the first prototype in the $N$ facets of the desired decision boundary.  See Figure \ref{fig:PolygonVoronoi} for an example construction with $N=5$, $d=2$. We formalise this observation as Proposition \ref{subset} below, and give a formal proof in the appendices.

\begin{figure}[htb]
 \centering
 \includegraphics[width=0.6\linewidth]{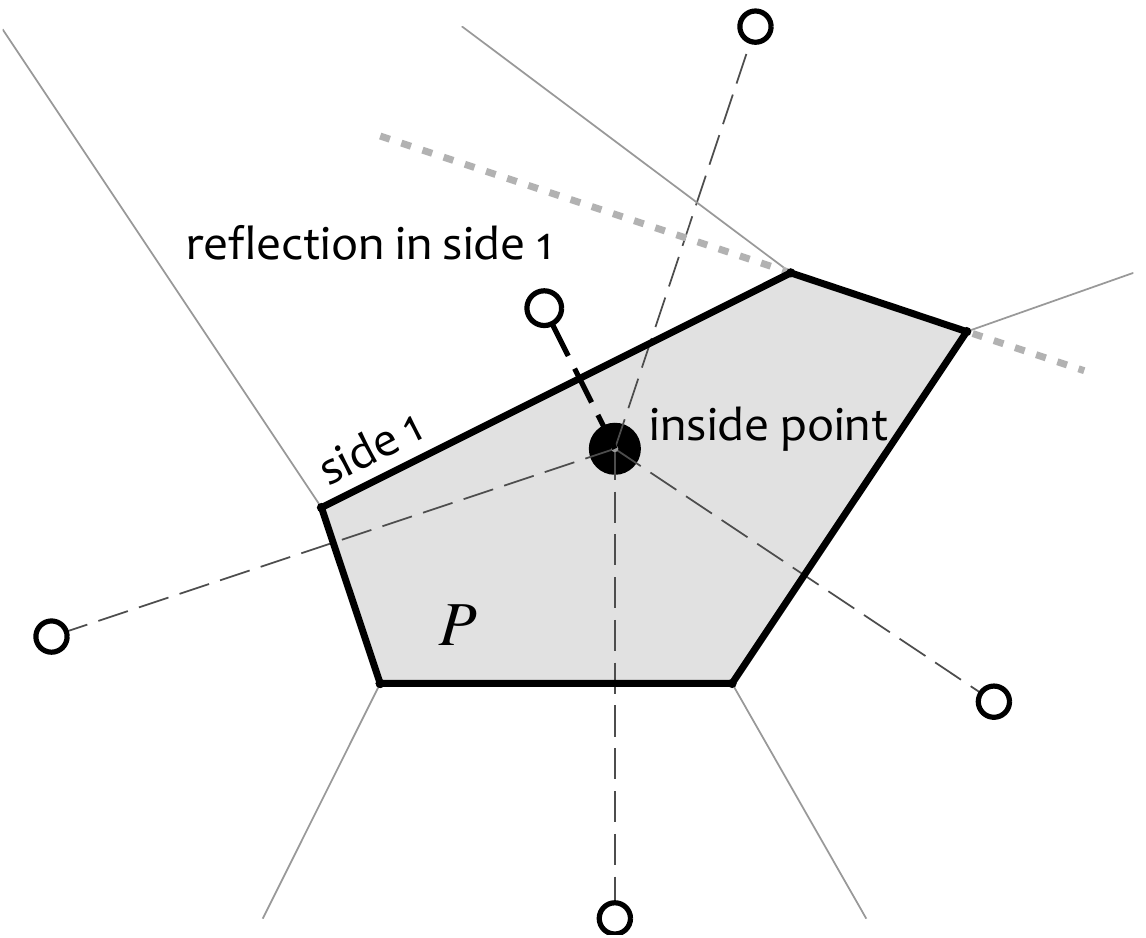}
 \caption{Illustration of the construction of a data set whose classification region is determined by a given convex polygon $P$.  The required prototype set (filled and empty circles) is constructed by reflecting an arbitrary interior point in (the lines containing) the edges of $P$. Voronoi boundaries are shown; the Voronoi cell of the interior point is the decision region of the classifier, coincident with $P$.}
 \label{fig:PolygonVoronoi}
 \end{figure}

We will use $G(d,N)$ to denote the set of classifiers $g:\mathbb{R}^d \to \{+1,-1\}$ whose decision boundary is a convex $N$-faceted polytope. 

\begin{lemma}
 \label{subset}
 The set of convex $N$-faceted polytope classifiers is a subset of the set of NN classifiers with reference set of size $N+1$, in the same Euclidean space.  That is, 
  \begin{equation}
 G(d,N) \subseteq \mathrm{1NN}(d,N+1).
 \end{equation}
\end{lemma}

\begin{corollary}
\label{impliesLowerBound}
A lower bound for the VC dimension of $G(d,N)$ is also a lower bound for the VC dimension of $\mathrm{1NN}(d,N+1)$.
\end{corollary}

The corollary follows immediately: any set of points which can be shattered by an element of $G(d,N)$ can be shattered by an element (the same element) of $\mathrm{1NN}(d,N+1)$.

Tak{\'a}cs gives the following result for two-dimensional Euclidean space:

\begin{lemma}
 (Tak{\'a}cs) 
 \label{Takacs1}
 \begin{equation}
 h(G(2,N)) \geq 2N + 2,
 \end{equation}
 for $N \geq 2$.
\end{lemma}

We give a sketch of the proof given in \cite{Takacs} in the appendices.  By a more complicated argument, Tak{\'a}cs establishes that this lower bound is also an upper bound for the VC dimension of polytope classifiers.  In general, upper bounds for polytope classifiers are of no relevance to the more general set of NN-rule classifiers.  But the $N=2$ case is an exception, and we make the following brief side remark about this case:

\begin{remark}
With three prototypes, the only decision surfaces an NN-rule classifier can form are an open 2-gon, or a pair of parallel lines.  But any finite set of points which can be correctly dichotomised by a pair of parallel lines can also be correctly split by a digon formed by a suitably small adjustment of the lines such that they are non-parallel.  So for this case, the set of NN-rule classifiers has no greater separating power than the related set of polygon classifiers.  Thus a precise value for the VC dimension of this set of NN-rule classifiers is established:

 \begin{equation}
 \mathrm{VCdim}(\mathrm{1NN}(2,3)) = 6.
 \end{equation}

\end{remark}
 
Now, for general $m$, Lemma \ref{Takacs1}, with Corollary \ref{impliesLowerBound}, implies a lower bound for $\mathrm{VCdim}(\mathrm{1NN}(2,m))$:

\begin{equation}
 \mathrm{VCdim}(\mathrm{1NN}(2,m)) \geq 2m
\end{equation}

However, we can do better than this.  For $m \geq 4$, the NN-rule classifier can create a larger class of decision surfaces than a single convex polytope. In Appendix~\ref{appendix1} we present a new argument, inspired by the elementary geometrical approach of Tak{\'a}cs but considerably more involved, demonstrating a stronger lower bound for $\mathrm{VCdim}(\mathrm{1NN}(2,m))$, $m \geq 4$:

\begin{proposition}
\label{OneBetter}
\begin{equation}
 \mathrm{VCdim}(\mathrm{1NN}(2,m)) \geq 2m + 1.
\end{equation}
\end{proposition}

Though this improvement is the smallest possible, it establishes the principle that the VC dimension of 1NN classifiers with $m$ prototypes in $\mathbb{R}^2$ is larger than that of the relevant comparable class of polygon decision boundaries (recall that Tak{\'a}cs' lower bound, Lemma \ref{Takacs1}, is also an upper bound for that class).  That is, Proposition \ref{OneBetter} establishes that 1NN classifiers in $\mathbb{R}^2$ gain in expressivity from their ability (for $m \geq 4$) to form boundaries other than convex polygons.

\section{Lower bound -- higher dimensions}
\label{lowerdd}
Tak{\'a}cs' result was extended by Tak{\'a}cs and Pataki to Euclidean spaces of dimension higher than two:

\begin{proposition}
 (\cite{TakacsPataki})
 \label{Takacs2}
 \begin{equation}
 \mathrm{VCdim}(G(d,N)) \geq dN + 2,
 \end{equation}
 for $d \geq 2$, $N \geq 2$.
\end{proposition}
\begin{proof}
 The geometrical arguments are less simple than for the two-dimensional case; we refer readers to \cite{TakacsPataki} for the proof.
\end{proof}
\begin{remark}
 Tak{\'a}cs and Pataki also offer slightly stronger lower bounds for the special cases $d=3$ and $d=4$: $h(G(3,N)) \geq 3N + 3$ and $h(G(4,N)) \geq 4N + 5$; but these require a higher minimum value for $N$ than Lemma \ref{Takacs2} does.
\end{remark}

As in the two-dimensional case, we can deduce a lower bound for the VC dimension of the NN classifier:
\begin{corollary}
  \label{LowerBoundProp}
  \begin{equation}
  \mathrm{VCdim}(\mathrm{1NN}(d,m)) \geq dm + 2 - d,
  \end{equation}
  for $d \geq 2$, $m \geq 3$.
\end{corollary}
\begin{remark}
The relation does not hold for $d=2$, $m=2$: an NN-rule classifier with two prototypes in the plane has for its decision boundary a line, so cannot shatter four points: linear classifiers are, famously, unable to solve the XOR problem.
\end{remark}

\section{Upper bounds}
\label{upper}

The \emph{shatter coefficient} of a family of sets $B$ (for our purposes, $B$ is a set of subsets of $\mathbb{R}^d$) is a number closely related to VC dimension.  It is called the ``growth function'' by some authors, but other authors give that term a different definition.  The $n$th shatter coefficient of $B$, denoted $S(B,n)$, is the maximum number of different subsets of $n$ points that can be formed by intersection of the $n$ points with elements of $B$: that is, the number of subsets that can be ``picked out'' using elements of $B$.  (The maximum is taken over all sets of $n$ points.)

The VC dimension of $B$ is then the largest $n$ such that $S(B,n) = 2^n$.  This is an expression of the concept of shattering for families of sets $B$ rather than classifiers: if the set of all subsets of the $n$ points which can be formed by intersection of the $n$ points with elements of $B$ is all $2^n$ possible subsets of the $n$ points, then the $n$ points are shattered by $B$.

The VC dimension of a family of classifiers, as we defined it in section \ref{introduction}, is equal to the VC dimension of the associated family of decision regions, as just defined, and the shatter coefficient of a family of classifiers is defined equal to the shatter coefficient of the associated family of decision regions see \citep[see][ch.12]{Devroye96}.  

Devroye et al. give upper bounds for the shatter coefficients of the class $C(d,m)$:
\begin{lemma}
 (Devroye, Gy{\"o}rfi, and Lugosi) For $m \geq 3$,
 \label{DevroyeLemma}
 \begin{align}
  S(C,n) &\leq 2^m n^{9(m-2)} &\text{for } d=2\\
  S(C,n) &\leq 2^m n^{(d+1)m(m-1)/2} &\text{for } d \geq 3.
 \end{align}
\end{lemma}
\begin{proof}
 See \cite[p. 312]{Devroye96}
\end{proof}
As before, $d$ is the dimension of the space and $m$ is the number of prototypes used by the classifier.  

These upper bounds are based simply on the observation that for a reference set with $m$ prototypes there are at most $m(m-1)/2$ Voronoi cell boundaries, so the number of points which can be shattered by the set of Voronoi diagrams with $m$ centres is bounded above by the number of points which can be shattered by $m(m-1)/2$ hyperplanes.  The stronger result for $d=2$ comes from a restriction on the number of edges of a planar graph, applied to the Delaunay triangulation which is the dual of the Voronoi diagram.

These bounds imply the following result for the VC dimension of the NN classifier:

\begin{proposition}
\label{Wbound}
For $m \geq 3$,
\begin{equation}
 \mathrm{VCdim}(\mathrm{1NN}(d,m)) \leq -\frac{q}{\log{2}} W_{-1} \left(-\frac{\log 2}{q} 2^{-\frac{m}{q}} \right),
 \end{equation}
where 
\begin{align}
  q &= 9(m-2) &\text{for } d = 2\\
  q &= (d+1)m(m-1)/2 &\text{for } d \geq 3.
\end{align}
$W$ is the Lambert W function; the $W_{-1}$ branch is the relevant branch. The logarithms are natural logarithms.
\end{proposition}

\begin{proof}
 See appendix \ref{LambertResult}.
\end{proof}

We can obtain from this a looser but more easily interpretable upper bound by using a recent result which gives a lower bound for the $W_{-1}$ branch of the Lambert function:
\begin{lemma}
(Chatzigeorgiou)
  \begin{equation}
    W_{-1}(-e^{-u-1}) > -1 - \sqrt{2u} - u,
  \end{equation}
  for $u>0$.
\end{lemma}
\begin{proof}
  See \cite{Chatz13}.
\end{proof}

 Using this result with proposition \ref{Wbound} gives the following looser bound on the VC dimension:
\begin{corollary}
\label{LooseCor}
 Let $q' = q/\log{2}$, where $q$ is defined as in proposition \ref{Wbound}. Then for $m \geq 3$,
\begin{equation}
 \label{looseupper}
 \mathrm{VCdim}(\mathrm{1NN}(d,m)) < q' \left( \sqrt{2\left(\frac{m}{q'} + \log{q'} -1\right)} + \frac{m}{q'} + \log{q'} \right).
\end{equation}
\end{corollary}
This upper bound enables us to bound the the asymptotic rate of growth of $C$.  Now, $q'$ grows monotonically with $m$ and with $d$, and grows at least as fast as $O(m)$ for increasing $m$.  So whether considering growth with increasing $m$ or growth with increasing $d$, the fastest-growing term within the brackets of (\ref{looseupper}) is $\log{q'}$.  So we have
\begin{equation}
\label{simq}
 \mathrm{VCdim}(\mathrm{1NN}(d,m)) \lesssim q' \log{q'},
\end{equation}
for large $m$ or $d$.  

Table~\ref{tab:results} summarises the VC dimension bounds derived in this study.

\section{Discussion of asymptotic behaviour}
\label{discussion}

Considering first the rate of growth of the VC dimension with $m$ for fixed $d$, equation \ref{simq} implies
\begin{equation}
\label{Om2}
\mathrm{VCdim}(\mathrm{1NN}(D,m)) \in O(m^2 \log(m^2)) = O(m^2 \log(m)),
\end{equation}
for fixed $d=D>2$, with the better result
\begin{equation}
\mathrm{VCdim}(\mathrm{1NN}(2,m)) \in O(m \log(m)),
\end{equation}
for $d = 2$. Figure~\ref{fig:VCBoundsFigure} illustrates the upper bounds as a function of $m$, for the cases $d = 2$ and $d = 3$.

\begin{figure}[htb]
\centering
\includegraphics[width=0.65\linewidth]{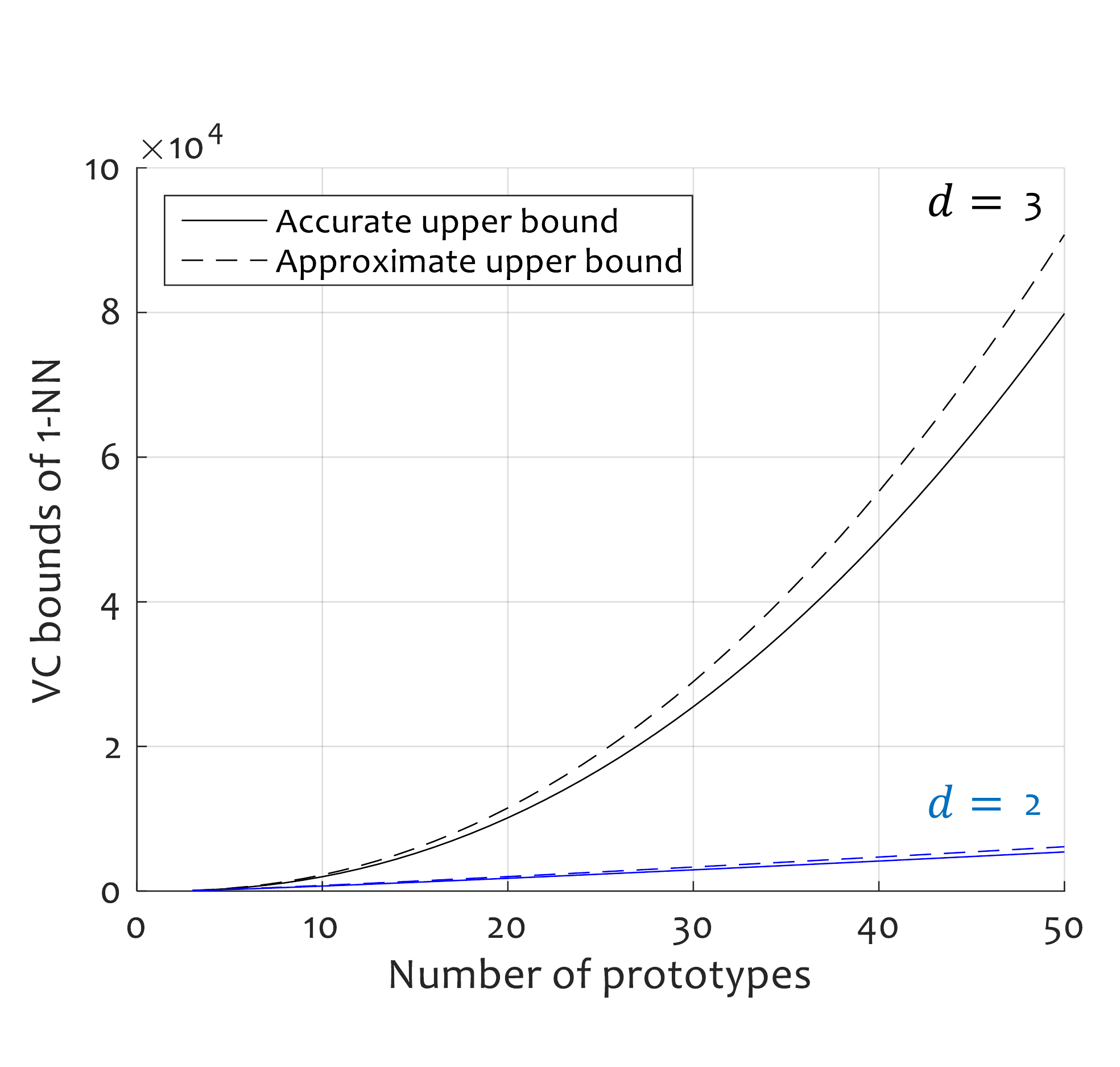}
\caption{Upper bounds for the VC dimension of NN-rule classifiers.  ``Accurate'' upper bounds are the best we have obtained, given in Proposition~\ref{Wbound}. ``Approximate'' upper bounds are the looser bounds given in Corollary~\ref{looseupper}.}
\label{fig:VCBoundsFigure}
\end{figure}

It is interesting to compare the log-linear growth in $m^2$ given by equation (\ref{Om2}) with recent results for neural networks given by \cite{Harvey}.  Consider a neural network with $d$ input neurons (the real coordinates of the feature space), and $N$ neurons in a single hidden layer with binary threshold activation functions. Each neuron in the hidden layer of such a network encodes a hyperplane decision boundary: the neuron will be in one or the other of its binary states depending on which side the input vector lies of a plane normal to the vector of weights of that neuron.  So a network with $N=\frac{1}{2}m(m-1)$ hidden neurons is comparable to 1NN classifiers with $m$ prototypes, in the sense that both build their decision boundaries from parts of $\frac{1}{2}m(m-1)$ hyperplanes.  \cite{Harvey} prove that a piecewise-linear neural network with $W$ parameters has VC dimension with asymptotic growth $O(W \log W)$ in the number of parameters.  For the network just described, the number of parameters is $ \sim Nd = \frac{1}{2}m(m-1)d$, meaning the asymptotic growth of its VC dimension is $O(m^2d \log m^2d)$, just as relation (\ref{simq}) gives as an upper bound for our 1NN classifiers.  That is, the neural network achieves the (asymptotically) highest value it can for its VC dimension, given the number of hyperplane decision surfaces it has to work with.  It is an interesting open question whether the same is true for 1NN classifiers.

Turning now to consider the rate of growth with $d$ for fixed $m=M$, equation (\ref{simq}) implies
\begin{equation}
 \mathrm{VCdim}(\mathrm{1NN}(d,M)) \in O(d \log(d)).
\end{equation}
Thus, the VC dimension grows polynomially in $d$ (asymptotically slower than $d^2$).  This has implications for learnability: for example, polynomial growth of the VC dimension of a class with the dimension of the space is a necessary condition for the class to be properly polynomially learnable~\cite[Theorem 3.1.1.]{Blumer89}.

\section{Conclusions}
\label{conclusions}
\begin{table}
\caption{Summary of the VC dimension bounds. $m$ denotes the number of prototypes in the reference set, and $d$ denotes dimensionality. $W_{-1}$ denotes the $-1$ branch of the Lambert W function.}
\label{tab:results}
\centering

\begin{tabular}{ccc}
Type of bound&Dimensionality&Expression\\
\hline
Lower&$d = 2$&$2m + 1$\\
Lower&$d \geq 2$&$dm + 2 - d$\\
Upper&$d = 2$&$-\frac{9(m-2)}{\log{2}} W_{-1} \left(-\frac{\log 2}{9(m-2)} 2^{-\frac{m}{9(m-2)}} \right)$\\
Upper&$d \geq 2$&$-\frac{(d+1)m(m-1)}{2\log{2}} W_{-1} \left(-\frac{2\log 2}{(d+1)m(m-1)} 2^{-\frac{2m}{(d+1)m(m-1)}} \right)$
\end{tabular}

\end{table}

The VC dimension for the set of all NN-rule classifiers in $d$-dimensional Euclidean space with a reference set of size $m$ grows at least as fast as $dm$ and not faster than $O(m^2 \log{m}$) as $m$ increases.  For the case of two-dimensional Euclidean space, the VC dimension grows not faster than $O(m \log{m}$).

Considering instead growth with $d$, the VC dimension for the set of all NN-rule classifiers in $d$-dimensional Euclidean space with a reference set of size $m$ grows at least as fast as $d(m-1)$ and not faster than $O(d \log{d}$) as $d$ increases. 

Precise lower and upper bounds for this VC dimension are given in our Corollary \ref{LowerBoundProp} and Proposition \ref{Wbound} respectively, and summarised in Table~\ref{tab:results}.

The consequence of these bounds that is of interest to practitioners is the implication for the size of sample needed to learn an accurate classifier.  In the Probably Approximately Correct learning model, the \emph{sample complexity} is the number of training examples needed to learn a classifier of given accuracy with given probability.  The sample complexity of a family of classifiers is known to depend linearly on the VC dimension~\cite[Theorem 6.8]{Shais}.  Therefore, the bounds we give above for the asymptotic growth of the VC dimension are also bounds on the asymptotic growth of the size of the training data set needed to learn (with given probability) an accurate NN classifier with reference set of given size.  The lower bound applies only to classification algorithms which produce reference sets of given fixed size (and can produce \emph{any} reference set of that size). The upper bound is significantly more broadly applicable. The upper bound applies to any NN classification algorithm which may not have more than $m$ points in its reference set.  

Thus we may conclude: the size of the training set needed to learn an accurate NN-rule classifier with reference set of size $m$ in $d$-dimensional Euclidean space grows not faster than $O(m^2 \log{m}$) as $m$ increases.  For the case of two-dimensional Euclidean space, the size of the training set required grows not faster than $O(m \log{m}$).  Considering instead growth with $d$ for fixed $m$, the size of the training set required grows not faster than $O(d \log{d}$).

The fact that the growth rate of the upper bound is asymptotically faster (with $m$) for $d \geq 3$ than for $d=2$ raises the interesting possibility that there may be something fundamentally different about the behaviour of NN-rule classifiers in 3 dimensions and higher from their behaviour in two-dimensional space.  If future work were to establish an $O(m^2)$ lower bound for $d \geq 3$, this would be confirmed. If, instead, an $O(m \log(m))$ upper bound were established, implying the same behaviour for the 1NN classifier in higher dimensions as in 2 dimensions, this would imply instead an interesting discrepancy between the behaviour of 1NN classifiers and the behaviour of neural networks with access to an equal number of hyperplanes from which to construct their decision boundaries, as discussed in section \ref{discussion}.

\section*{Acknowledgment}
This work was done under project RPG-2015-188 funded by The Leverhulme Trust, UK.  While preparing the paper for publication, IG received support from the European Union's Horizon 2020 research and Innovation programme under grant agreement No 731593.

\bibliographystyle{natbib}

\appendix
\section{Proof of Lemma \ref{subset}}
\begin{proof}
 We will show by construction that for an arbitrary convex polytope $P$ with $N$ facets, there is a labelled set of $N+1$ points such that an NN classifier using this set as a reference set will have $P$ for its decision boundary.   

 The reference set is constructed as follows: place one point, with the label of the interior region, anywhere in the interior of the polytope $P$.  The remaining $N$ points, with the opposite label, are placed as the reflections of this point in the $N$ hyperplanes which contain the $N$ facets of $P$.  The $i$th hyperplane is therefore the locus of points equidistant from the interior point and the $i$th exterior point. As the polytope is convex, no part of the interior of the polytope lies on the side of the $i$th hyperplane closer to the exterior point (in particular, the placement of the $i$th exterior point does not impact any boundary facet other than that formed by the $i$th hyperplane). Conversely, if a point is on the interior side of all $N$ hyperplanes, then it is in the interior of the polytope. 
 
 Thus, by construction, the set of points closer to the interior point than to any exterior point is the intersection of the $N$ half-spaces which are closer to the interior point than to the respective exterior points.  That is, the Voronoi cell of the interior point is the polytope $P$: the decision boundary using the NN classifier is $P$.
 
 Each convex $N$-faceted $d$-tope classifier is therefore a NN classifier with a reference set of size $N+1$.  The set of all such polytope classifiers is therefore a subset of the set of all such NN classifiers.
\end{proof}

\section{Proof of Lemma \ref{Takacs1}}
\label{TakacsAppendix}
\begin{proof}
 To prove that the VC dimension of the $N$-gon classifier is at least $2N+2$, we must show that there exists an arrangement of $2N+2$ points which can be correctly labelled using an $N$-gon decision boundary, for any possible partitioning of the points into two classes (i.e., the $2N+2$ points can be shattered).
 
 The $2N+2$ points to be labelled are arranged as follows: place $2N + 1$ points on a circle, with the final point in the centre.  Then an $N$-gon can be constructed to include any subset of the points on the circle, together with the centre point, excluding the other points. The worst case is that, going round the circle, there are $N$ separate groups of points (of size one or two) of the opposite label to the centre point; the $N$ edges of the $N$-gon may be arranged to exclude one of these $N$ groups each.  It is always possible to place a line separating a sequence of points around the circle from the rest of the points on the circle, as the points on the circle are in convex position.
 
 The case where a sequence of points of opposite label to the centre point extends more than half-way round the circle, as in Figure~\ref{fig:Takacs}(a), must be considered separately but presents no difficulties.
\end{proof}

\begin{figure}[htb]
\centering
\includegraphics[width=0.9\linewidth]{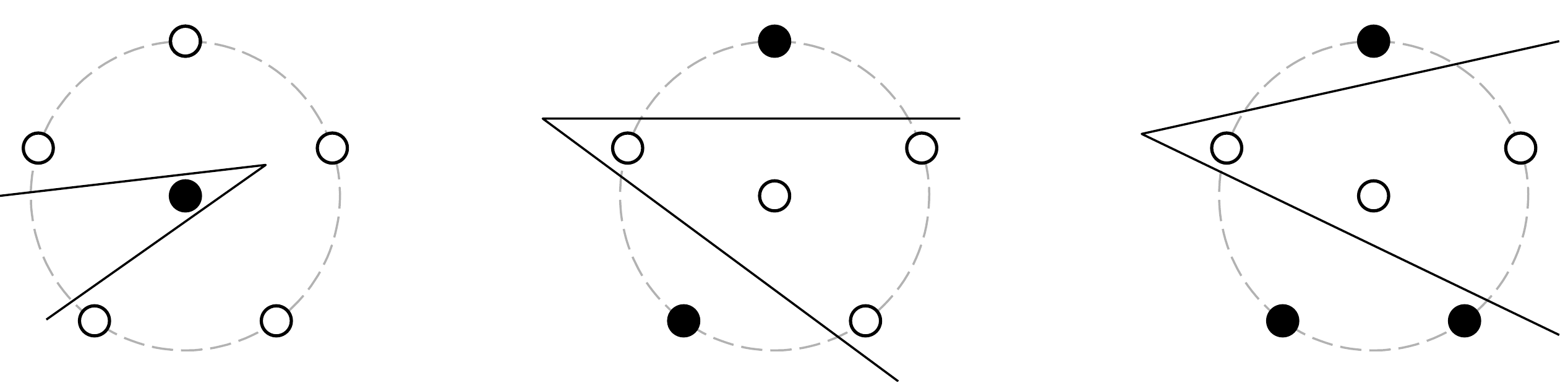}
(a)\rule{2.7 cm}{0 cm}(b)\rule{2.7 cm}{0 cm}(c)
\caption{Illustration of an arrangement of 6 points which can be shattered by an unbounded ``polygon'' with $N=2$ sides.  Plots (a), (b), and (c) illustrate examples of, respectively, a 1/5, a 2/4, and a 3/3 partition.}
\label{fig:Takacs}
\end{figure}

\begin{remark}
 In the $N=2$ case the decision boundary is two half-lines radiating from one point; in Tak{\'a}cs' terminology a ``polygon'' need not be bounded. Figure~\ref{fig:Takacs} illustrates examples of partitionings in the case of $N=2$ and six points. Any partition of the points into two groups can be constructed with the V-shaped border. 
\end{remark}

\begin{remark}
 \citet[p. 224]{Devroye96} make a similar construction with polygons and points on a circle, but use it only to prove that the class of all convex polygons has infinite VC dimension.
\end{remark}

\section{Proof of proposition \ref{OneBetter}}
\label{appendix1}
We begin with some geometrical preliminaries, as illustrated in Figure~\ref{fig:PolygonGeometry}.  A diagonal of a polygon is a line segment joining two non-adjacent vertices.  The vertices of a regular polygon all lie on the same circle, called the circumcircle of the polygon, and are equally spaced around it.  Each vertex of a regular $(2m-1)$-gon (for $m \geq 3$) is associated with two ``longest diagonals'', which connect the vertex in question to the vertices which are furthest away from it (both around the perimeter and in Euclidean metric). The two longest diagonals are reflections of each other in the diameter of the circumcircle which passes through the given vertex.

\begin{figure}[htb]
\centering
\includegraphics[width=0.6\linewidth]{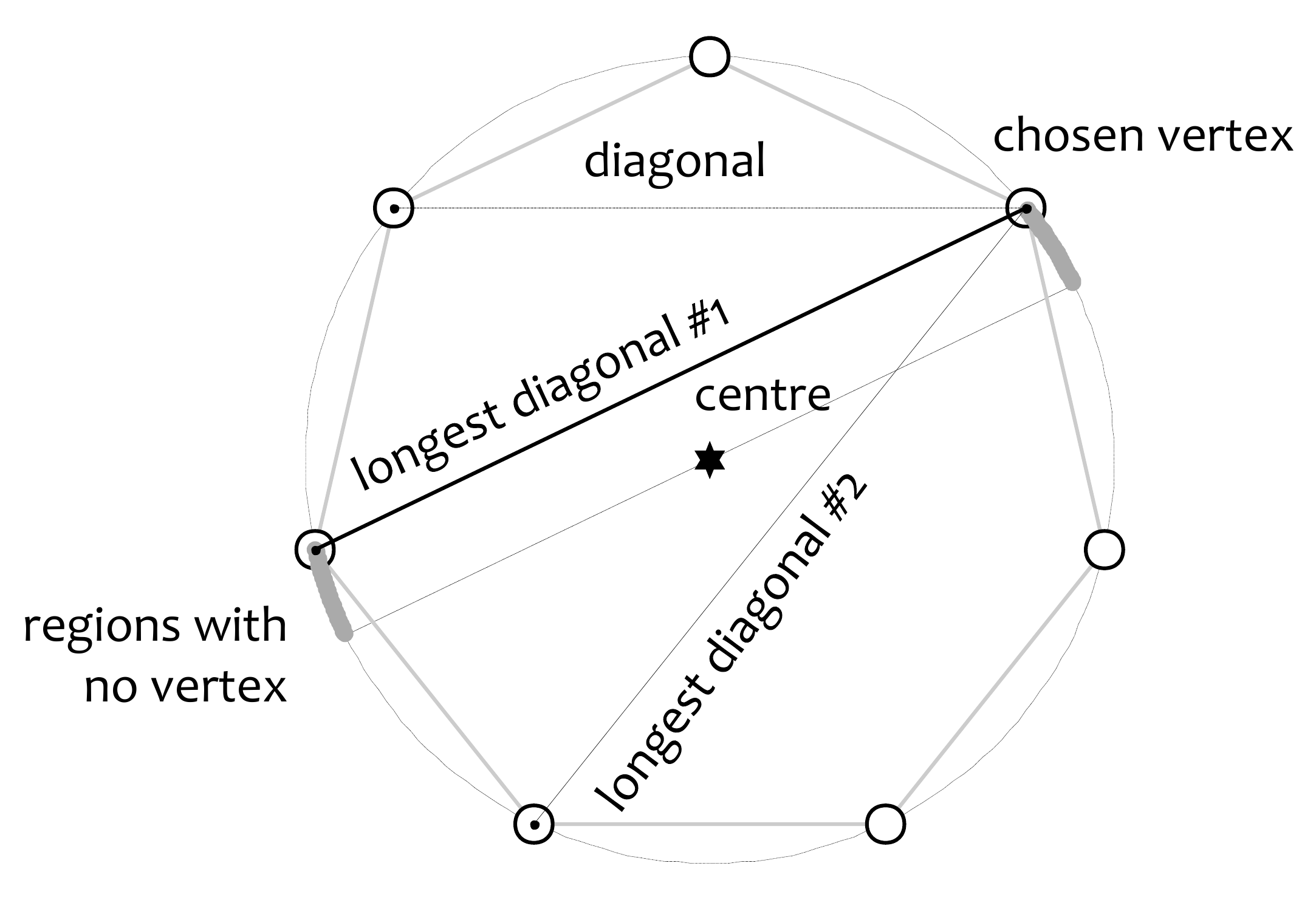}
\caption{Geometrical illustration for Lemma~\ref{geometry1}.}
\label{fig:PolygonGeometry}
\end{figure}

\begin{lemma}
 \label{geometry1}
 No vertices of a regular $(2m-1)$-gon lie between the line of a longest diagonal of the $(2m-1)$-gon and a line through the centre of the polygon which is parallel to this diagonal.
\end{lemma}
\begin{proof}
Recall that vertices of a regular polygon must all lie on the same circumcircle. If there were a vertex between the longest diagonal and a diameter of the circumcircle which does not intersect it, then a line from this vertex to the end of the longest diagonal further from it would be a diagonal longer than the longest diagonal.
\end{proof}

We will now introduce the arrangement of points which we will subsequently argue can be shattered by a NN-rule classifier with $m$ prototypes. 

\begin{arrangement}
  \label{arrangement1}
 (Figure~\ref{LemmaArrangement}.) Place $2m-1$ points as the vertices of a regular $(2m-1)$-gon, $m \geq 4$.  Two further points are placed on a line which runs through the centre of the $(2m-1)$-gon, and which is perpendicular to a line from one vertex (marked `$A$') through the centre.  The two points are placed on this line at an equal distance either side of the centre. This distance is sufficiently small that neither of the two points is separated from the centre by any of the diagonals of the $(2m-1)$-gon. This is equivalent to requiring that the two points lie within the smaller $(2m-1)$-gon formed by the central sections of the longest diagonals of the original $(2m-1)$-gon, between their mutual intersections.
\end{arrangement}

\begin{figure}[htb]
\centering
\includegraphics[width=0.55\linewidth]{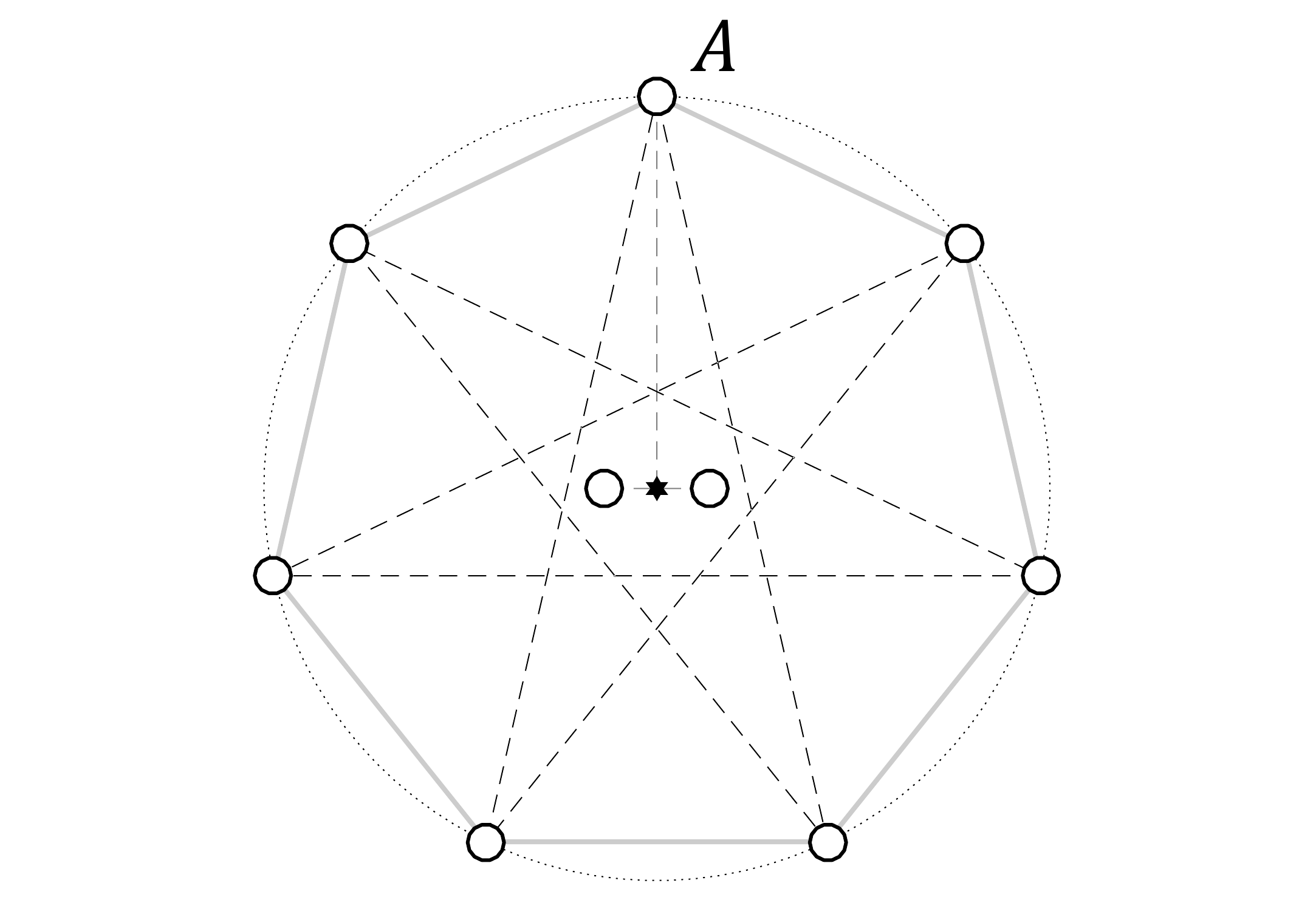}
\caption{Arrangement of 9 points which can be shattered by a NN-rule classifier with $m=4$ prototypes, but not by the set of classifiers whose decision boundary is a triangle.  7 of the points are vertices of a (convex) regular heptagon, which is shown together with its circumcircle and longest diagonals.  Also shown are segments of two lines used in the construction of the remaining two points, and the centre point.}
\label{LemmaArrangement}
\end{figure}

We now have two lemmata which establish some geometric properties of this arrangement of points.  Lemma \ref{arrangement} will show that there are certain sets of three points, which include one of the two centre points but not the other, which can be separated from the remaining points by two parallel lines.  This will be important when we come to demonstrate that the set of points can be shattered, because the labellings which are difficult to achieve are those for which the central points do not both have the same label.
\begin{lemma}
\label{arrangement}
Let $2m+1$ points be placed as Arrangement~\ref{arrangement1}, $m \geq 4$.  Call one of the inner points $B$, and the other $W$.  Let $C$ be the vertex nearest to $W$. Then, given any choice $D$ of vertex other than $C$, there exist vertices $D_1, D_2$, not adjacent to $D$, such that $\{D,D_i,B\}$ can be separated from all other points by two parallel lines. (Figure~\ref{fig:LemmaArrangementParallelLines}.)  
\end{lemma}

\begin{figure}
\centering
\begin{tabular}{cc}
\includegraphics[width=0.5\linewidth]{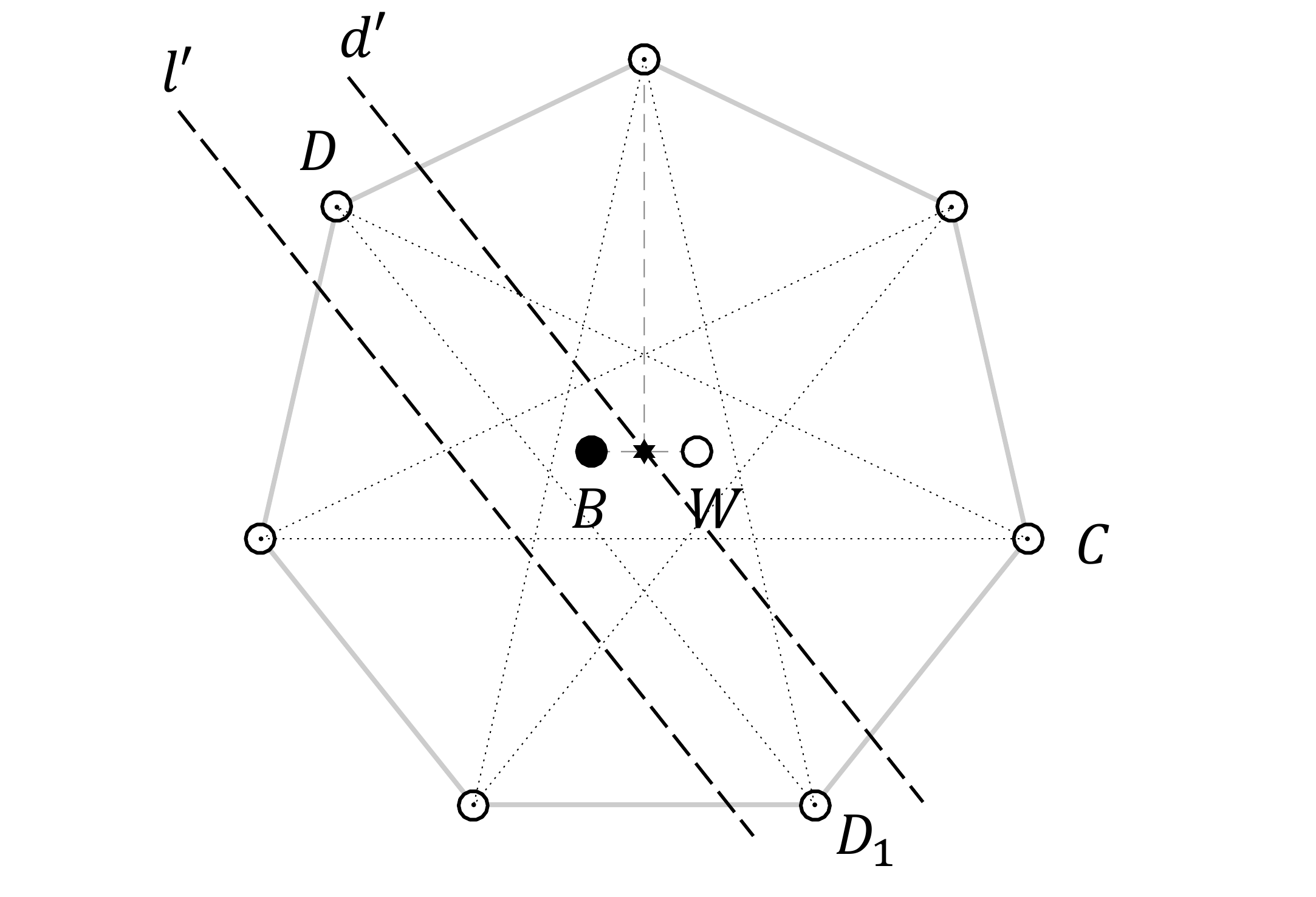}&
\includegraphics[width=0.5\linewidth]{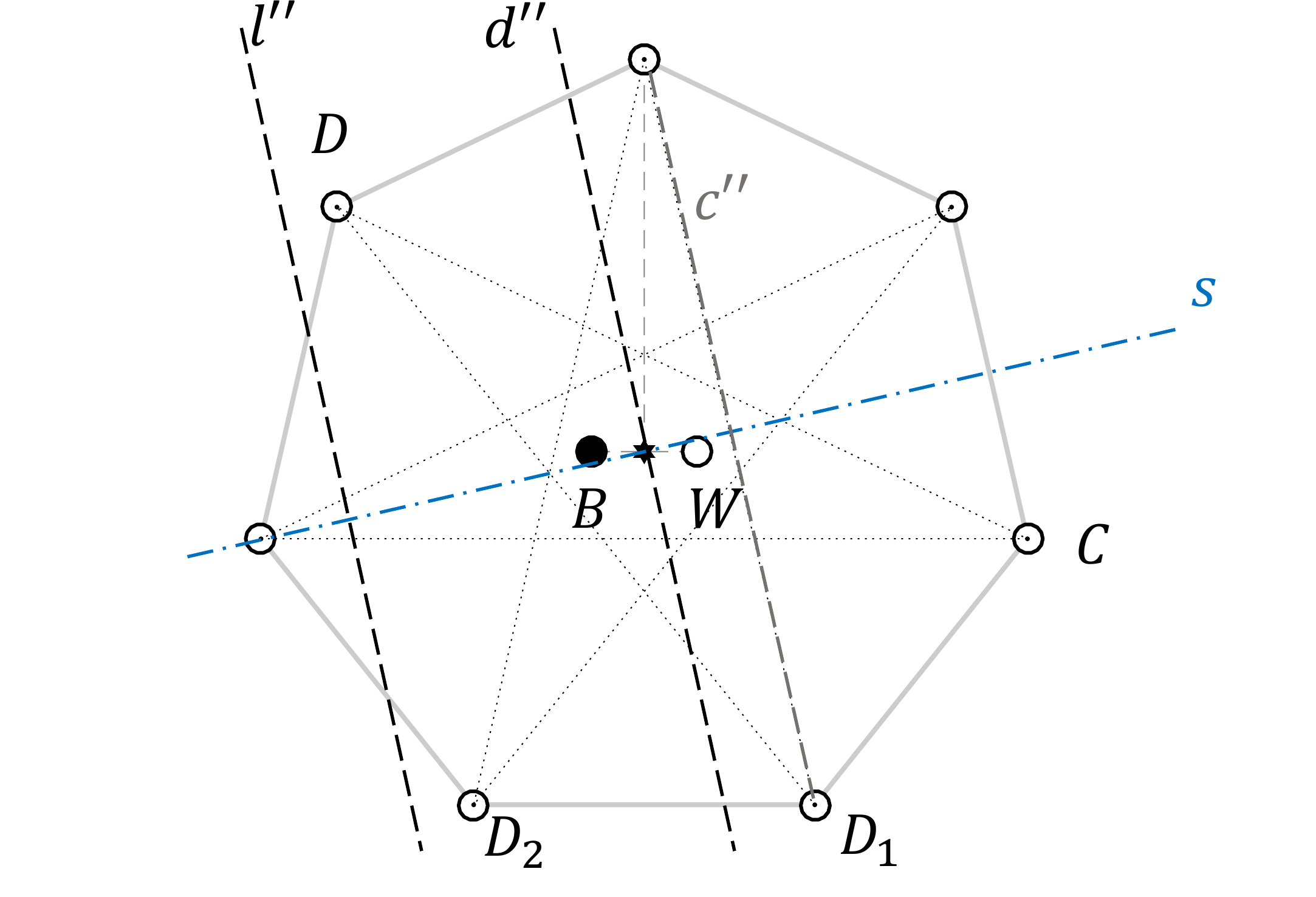}\\
(a) $\{D,B,D_1\}$ separated.&(b) $\{D,B,D_2\}$ separated.\\
\end{tabular}
\caption{Given a choice B from the two centre points, and a choice D of any vertex apart from C, there are two sets of three points, containing the two chosen points, which can be separated from the remaining points by two parallel lines. The heptagon illustrated in (b) shows the worst case for the separation of these lines, at just under two-thirds of the circumradius.  The lines may be placed closer together if $m$ is greater or if $D_1$ is chosen instead of $D_2$.  Note that $l'$ (resp. $l''$) may be placed arbitrarily close to diagonal  $DD_1$ (resp.\ $DD_2$).}
\label{fig:LemmaArrangementParallelLines}
\end{figure}


\begin{proof}
 Denote the diameter of the circumcircle which passes through point $D$ by $d$. One of the two longest diagonals from $D$ passes on the same side of $d$ as $B$, and the other passes on the side of $W$.  One of the two points which we will prove has the desired property is the point at the opposite end of that longest diagonal from D which passes on the side of $B$; call this point $D_1$.

We will construct two parallel lines such that $D$, $B$, and $D_1$ are between the lines, and the $2m-2$ remaining points are not (Figure~\ref{fig:LemmaArrangementParallelLines}(a)).  The required lines are parallel to $DD_1$. One of the lines is a diameter, $d'$.  By Lemma \ref{geometry1} there is no vertex between $d'$ and $DD_1$. The second parallel line, $l'$, is placed on the far side of $DD_1$ from the centre. It can be placed arbitrarily close to $DD_1$, so it is always possible to place it such that that no vertices lie between $l'$ and $DD_1$.  Therefore the only points between $d'$ and $l'$ are $B$, $D$, and $D_1$. 
 
$D_2$ is the vertex adjacent to $D_1$ on the side of $B$ (Figure~\ref{fig:LemmaArrangementParallelLines}). 
  
The set $\{D,D_2,B\}$ is separated from the rest of the points in the construction by two parallel lines as follows.  Consider the line of symmetry  $s$ of the polygon for which $D_2$ is the reflection of $D$.  The two lines required are perpendicular to $s$ (unless $s$ passes through the vertex defined in Arrangement \ref{arrangement1} as $A$, which case we discuss shortly). The first line, $d''$ passes through the centre of the polygon. The second line, $l''$, is placed further from the centre than $DD_2$, a sufficiently small distance from $DD_2$ that no vertices lie between it and $DD_2$. 

We now argue that no vertices lie between $d''$ and $DD_2$. Consider the longest diagonal from $D_1$ which does not go to $D$.  Call this longest diagonal $c''$.  Now suppose there were a vertex on the arc of the circumcircle between $D_1$ and $D_2$: then $D_2$ would not be an adjacent vertex to $D_1$, which it is by definition.  Suppose instead there were a vertex on the arc of the circle between $D$ and the vertex at the other end of $c''$ from $D_1$.  Then a line segment from this vertex to $D_1$ would be a diagonal longer than the longest diagonal.  Therefore there are no vertices between $c''$ and $DD_2$.  Now, $d''$ is parallel to $DD_2$, therefore, by the symmetry of the regular $(2m-1)$-gon in $s$, it is parallel to $c''$, and it is closer to $DD_2$ than $c''$ is, since it passes through the centre.  So the region between $d''$ and $DD_2$ is entirely contained in the region between $c''$ and $DD_2$, so the fact that there is no vertex in the latter region implies there is no vertex in the former region.
 
If $BW$ is perpendicular to $s$, then $d''$ and $l''$ are placed at a suitably small angle (it may be arbitrarily small) to the perpendicular to the line of symmetry, such that $B$ falls between the lines and $W$ does not. 
 
\end{proof}
 
We shall need to know the width of the strip constructed in the previous lemma to contain $D$, $B$, and $D_2$.
 
\begin{lemma}
\label{widthlemma}
 Let $2m+1$ points be placed as Arrangement~\ref{arrangement1}, $m \geq 4$.  Let $D$ and $D_2$ be defined as in Lemma \ref{arrangement} (Figure~\ref{fig:LemmaWidth}).  Then the distance from the centre of the polygon to $DD_2$ is less than 0.63 times the circumradius.
\end{lemma}

\begin{figure}[htb]
\centering
\includegraphics[width=0.7\linewidth]{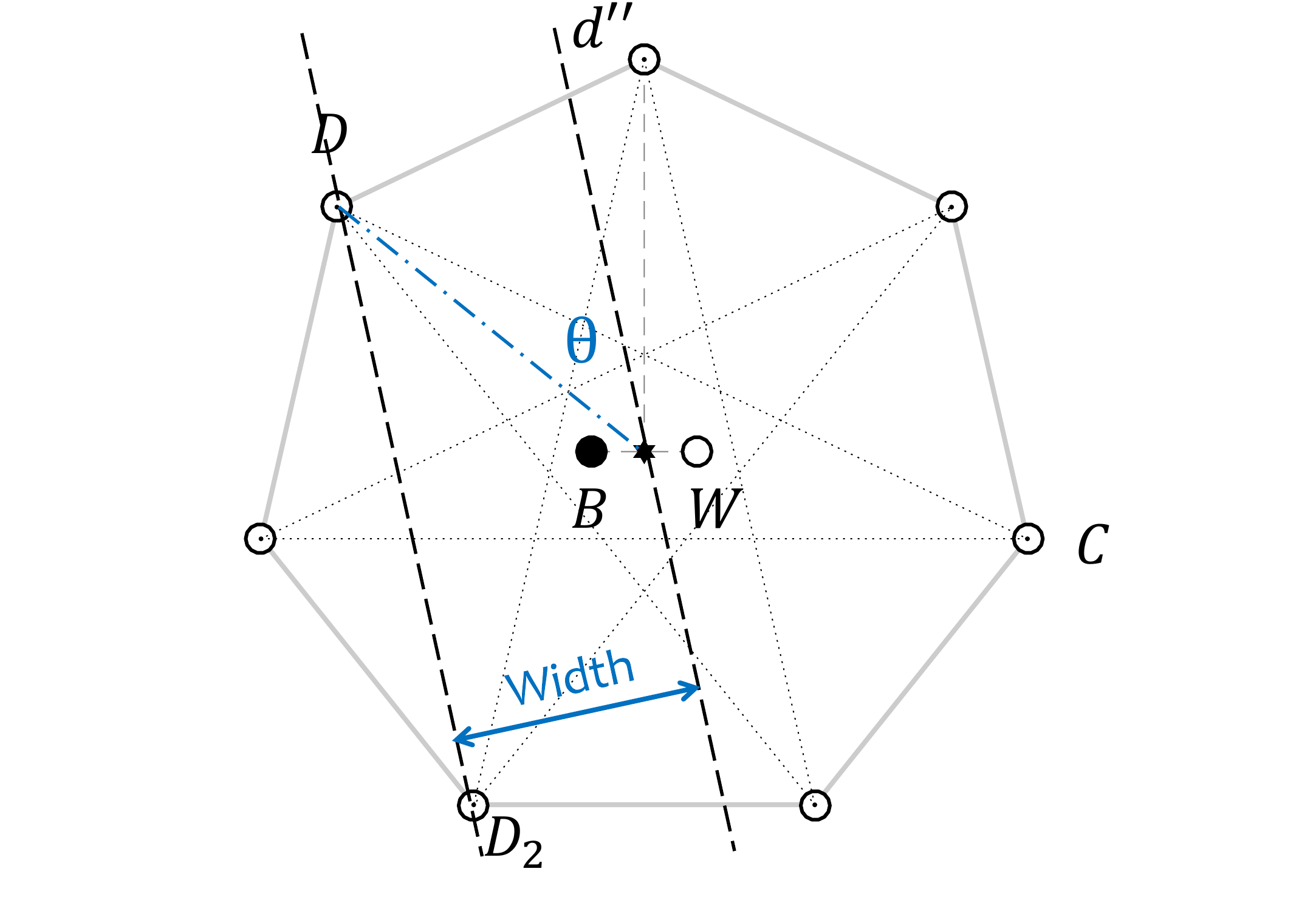}
\caption{Greatest possible distance from $DD_2$ to centre}
\label{fig:LemmaWidth}
\end{figure}

\begin{proof}
 By an elementary geometrical argument, the angle between the line passing through the centre and $D$, and $d''$, is 
 \begin{equation}
  \theta = \dfrac{3 \pi}{2(2m-1)}.
 \end{equation}
Therefore the distance between $DD_2$ and the centre is
\begin{equation}
 \mathrm{Width} = R\sin\dfrac{3 \pi}{2(2m-1)},
\end{equation}
where $R$ is the circumradius.  This width decreases monotonically with $m$.  The largest value it may take in our construction is therefore associated with the smallest value of $m$ we consider, $m=4$.  In this case,
\begin{equation}
 \mathrm{Width} = R\sin \dfrac{3 \pi}{14} < 0.63R.
\end{equation}

\end{proof}

We are now in a position to prove the proposition.

\begin{proposition}
 \begin{equation}
  \mathrm{VCdim}(\mathrm{1NN}(2,m)) \geq 2m + 1
 \end{equation}
for $m \geq 4$.
\end{proposition}
\begin{proof}
 We shall prove that the set of $2m+1$ points placed as Arrangement~\ref{arrangement1} can be shattered by NN-rule classifiers using $m$ prototypes.  That is, we will prove that all labellings of these points can be correctly classified by a NN-rule classifier with $m$ prototypes in the reference set.  
 
 If the two inner points have the same label, then the situation is the same as described for the polygon classifier in Appendix \ref{TakacsAppendix}: an $(m-1)$-gon can be constructed containing the two central points and all the points on the circle of the same label, just as if there were only one central point.
 
 Consider now the case where the two interior points are of different labels.  We will refer to the class less-represented among a given labelling of the $2m+1$ points as ``black'', and the other class as ``white''.  The black class, by this definition, contains at most $m$ points.  It remains to prove that all labellings in which the interior points have different colours can be labelled by a NN-rule classifier with $m$ prototypes.
 
 Identify the black and white interior point with the points B and W respectively of Lemma~\ref{arrangement}.  Point C is then defined as in that lemma.
 
 Suppose there are $m$ points in the ``black'' class, of which $m-1$ are vertices of the polygon.  Since $m \geq 4$, two of the $m$ black points must then be vertices of the polygon other than C.  Let one of these these vertices be P, and define $P_1$ and $P_2$ to have the same relation to P as $D_1$ and $D_2$ have to $D$ in Lemma \ref{arrangement} (Figure~\ref{fig:PropositionSeparation}).  If either $P_1$ or $P_2$ is black, then that point, with P and the black centre point, can be separated from the remaining points by two parallel lines as described in Lemma \ref{arrangement}.  If both $P_1$ and $P_2$ are black, pick one of them arbitrarily to make the construction. By Lemma \ref{widthlemma}, the distance between these parallel lines is less than two thirds of the circumradius.  These parallel lines can therefore be achieved as a decision surface by three prototypes all of which are inside the circle: a black prototype is placed on 
 the line of the symmetry of the polygon which is perpendicular to these two lines, halfway between them, and therefore less than one third of the circumradius away from the centre of the circle.  Two white prototypes are placed as the reflections of this black prototype in the two parallel lines; since the black prototype is less than one-third of the radius away from the centre, the white prototypes are within the circle.  Each of the remaining $m-3$ black vertices can be separated from the rest of the points by a single line (due to the convexity of the polygon, as in the argument for the polygon classifier).  This line can be achieved as a decision surface by placing a black prototype as the reflection in this line of the nearer white prototype. All vertices not so separated and not contained between the two parallel lines first constructed will be classified as white by the two white prototypes.  Thus the $2m+1$ points are classified correctly by $m$ prototypes, 3 within the polygon and $m-3$ outside it.
 
\begin{figure}
\centering
\includegraphics[width=0.65\linewidth]{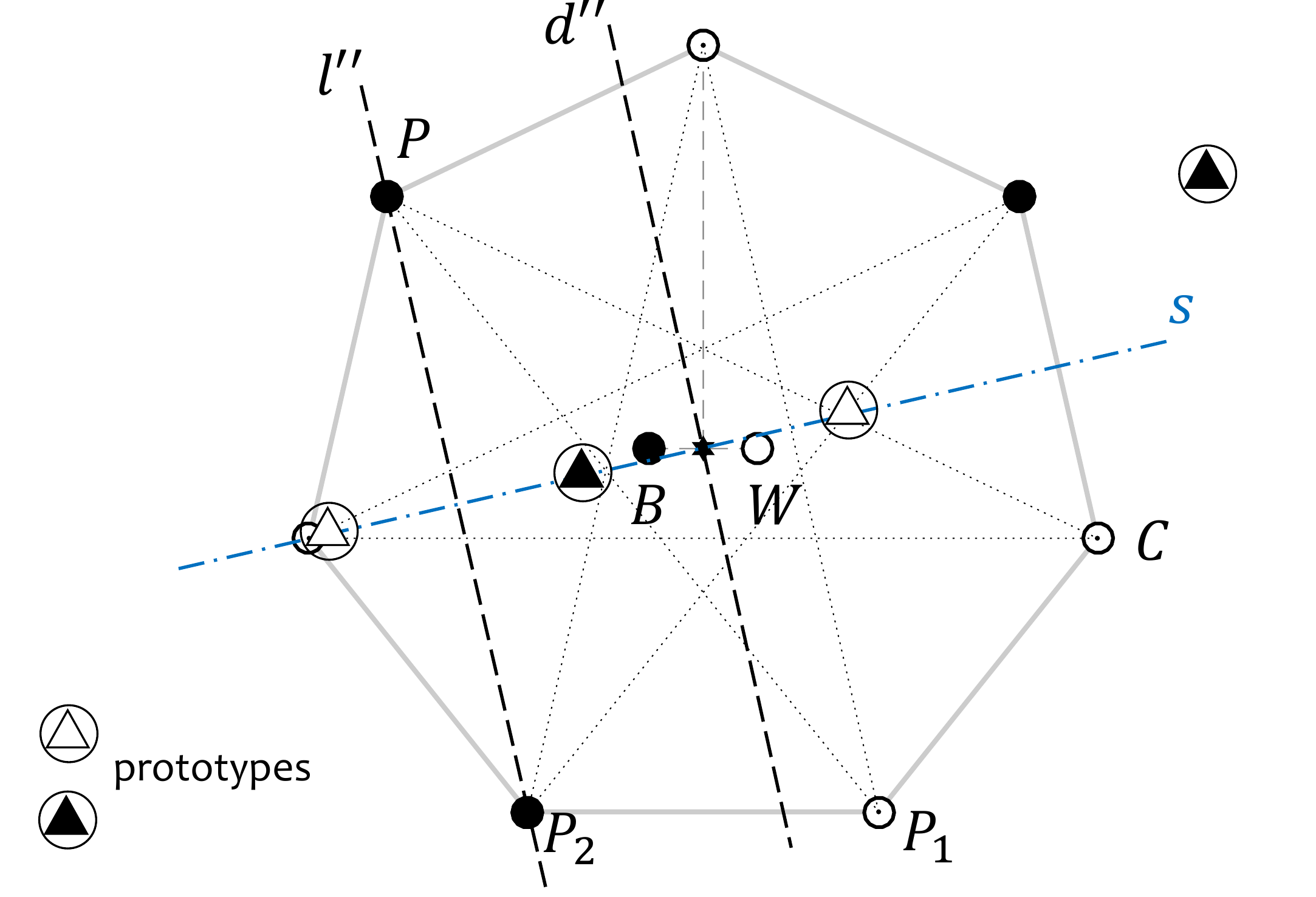}
\caption{Example labelling of 9 points by 4 prototypes. $P$, $B$, and $P2$ are labelled by a black prototype inside the circle.  They are separated from all other points by parallel decision surfaces created by two white prototypes which are also inside the circle.  The remaining black vertices (one such, in this case) are labelled by one black prototype each, placed outside the circle so as to be the nearest prototype only for the relevant vertex.}
\label{fig:PropositionSeparation}
\end{figure}

Now consider the case in which neither $P_1$ nor $P_2$ is black. In this case, $m-1$ black labels must be distributed among the $2m-3$ vertices which are neither $P_1$ nor $P_2$.  Consider the case where no two adjacent vertices are both black.  If no two adjacent vertices are both black, then there is only one possible arrangement (since $P_1$ and $P_2$ are adjacent): The vertices either side of the $\{P_1,P_2\}$ pair must both be black, and every second vertex between them (going around the circle) must be black.  Let $Q$ be a black vertex in this arrangement which is neither $P$ nor $C$ (recall that there must be at least one such vertex, because $m \geq 4$).  Let the two points $Q_1$, $Q_2$ be the two points which have the same relation to $Q$ as $D_1$ and $D_2$ have to $D$ in the proof of Lemma \ref{arrangement}.  At most one of $Q_1$ and $Q_2$ can be an element of $\{P_1,P_2\}$, because $Q$ is not $P$.  Therefore, since $Q_1$ and $Q_2$ are adjacent vertices, one of them must be black, and can therefore be separated along with $Q$ and the black centre point by two parallel lines, allowing correct classification of all the points as in the case where $P_1$ or $P_2$ was black.  Thus, either the points can be classified this way, or there are two adjacent vertices which are both black.  
 
Two cases now remain to be considered.  First, the case in which neither $P_1$ nor $P_2$ is black, but the remaining vertices do not have the only labelling which ensures no two adjacent vertices are black.  Second, the case in which there are fewer than $m$ points in the black class.  The construction is essentially the same for both of these cases.
The construction proceeds as previously, but it is only necessary to contain two black points between the parallel lines.  $B$ and one of the black vertices not part of the adjacent pair (if there is only one pair, or not the common vertex between two pairs if there are precisely three black vertices forming a run of three) are separated by two parallel lines (an arbitrarily small distance either side of the line containing these two points); these lines are achieved as a decision surface by three prototypes within the circle, a black prototype between two white prototypes as before.  The remaining $m-3$ prototypes are black, and are placed as reflections of the nearer white prototype in a line separating a single black vertex or an adjacent pair of black vertices from the other points.  Not all $m-3$ prototypes may be needed for this purpose, but unneeded prototypes may be placed at a large distance from the points so as not to affect the classification.
\end{proof}

\section{Proof of Proposition \ref{Wbound}}
\label{LambertResult}
\begin{proof}
 From the definition of VC dimension in terms of shatter coefficient, $h(C)$ is given by the largest $n$ solving
  \begin{equation}
  S(C,n) = 2^n.
  \label{VCdef}
 \end{equation}
 Now, if a continuous function $f(n)$ satisfies 
 \begin{equation}
 f(n) \geq S(C,n) \quad \forall n,
 \end{equation}
 and also $f(n)$ grows more slowly than $2^n$ asymptotically, then the largest $n$ solving 
 \begin{equation}
  f(n) = 2^n
  \label{feq}
 \end{equation}
 must be at least as large as the largest $n$ solving (\ref{VCdef}).  The bounds of Lemma \ref{DevroyeLemma} satisfy these conditions, so we may substitute them into (\ref{feq}) to find that $h(C)$ is bounded above by the largest $n$ solving
 \begin{equation}
  2^m n^q = 2^n.
  \label{VCeq}
 \end{equation}
 Equation (\ref{VCeq}) is solved by
 \begin{equation}
  n = -\frac{q}{\log{2}} W \left(-\frac{\log 2}{q} 2^{-\frac{m}{q}} \right),
  \label{Wresult}
 \end{equation}
 where the solution on the $W_{-1}$ branch is the largest real solution, which is therefore the upper bound.
\end{proof}

\end{document}